\pdfoutput=1

\documentclass[conference]{IEEEtran}
\IEEEoverridecommandlockouts

\usepackage{amssymb}
\usepackage{cite}

\ifCLASSINFOpdf
  \usepackage[pdftex]{graphicx}
\else
\fi
\usepackage{amsmath}
\newtheorem{proposition}{Proposition}
\newtheorem{defination}{Defination}
\newtheorem{remark}{Remark}
\newtheorem{proof}{Proof}
\usepackage{float}
\ifCLASSOPTIONcompsoc
  \usepackage[caption=false,font=normalsize,labelfont=sf,textfont=sf]{subfig}
\else
  \usepackage[caption=false,font=footnotesize]{subfig}
\fi

\usepackage{booktabs}
\usepackage{epstopdf}
\usepackage{enumerate}

\begin{document}
%
\title{Poisson Conjugate Prior for PHD Filtering based Track-Before-Detect Strategies in Radar Systems}



\author{\IEEEauthorblockN{Haiyi Mao,
Cong Peng,
Yue Liu,
Jinping Tang,
Hua Peng and
Wei Yi
}
\IEEEauthorblockA{School of Information and Communication Engineering, \\
University of Electronic Science and Technology of China,
Chengdu, China \\
Email: maohaiyi@std.uestc.edu.cn,
panycg\_cn@163.com,
rerrain521@gmail.com,  \\
jptang@std.uestc.edu.cn,
202121010920@std.uestc.edu.cn,
kussoyi@gmail.com}
}



\maketitle

\begin{abstract}
A variety of filters with track-before-detect (TBD) strategies have been developed and applied to low signal-to-noise ratio (SNR) scenarios, including the probability hypothesis density (PHD) filter.
Assumptions of the standard point measurement model based on detect-before-track (DBT) strategies are not suitable for the amplitude echo model based on TBD strategies.
However, based on different models and unmatched assumptions, the measurement update formulas for DBT-PHD filter are just mechanically applied to existing TBD-PHD filters.
In this paper, based on the Kullback-Leibler divergence minimization criterion, finite set statistics theory and rigorous Bayes rule, a principled closed-form solution of TBD-PHD filter is derived.
Furthermore, we emphasize that PHD filter is conjugated to the Poisson prior based on TBD strategies.
Next, a capping operation is devised to handle the divergence of target number estimation as SNR increases.
Moreover, the sequential Monte Carlo implementations of dynamic and amplitude echo models are proposed for the radar system.
Finally, Monte Carlo experiments exhibit good performance in Rayleigh noise and low SNR scenarios.
\end{abstract}


%
\IEEEpeerreviewmaketitle

\section{Introduction}
Multi-target detection and tracking with a low signal-to-noise ratio (SNR) is always a difficult problem for radar systems \cite{lowSNR2022}.
Under this circumstance, the traditional detect-before-track (DBT) strategy based on threshold will result in poor tracking performance.
Instead, raw data from echo is directly used to establish likelihood functions with track-before-detect (TBD) strategies \cite{Jiang2016kbTBD}.
This significantly reduces information loss and improves tracking performance, whose implementations include batch processing \cite{Yi2020tbdManeuvering} and particle filter (PF) \cite{Salmond2001pfTBD, Boers2004pfTBD, Ubeda2017pfTBD}.

Generally, a complete Bayes recursive filter is computationally intractable \cite{BarShalom2004estimation, Blackman1986radar}.
Recently, Mahler develops finite set statistics (FISST) theory to jointly estimate the number and states of targets \cite{Mahler2007FISST}.
Multi-target tracking based on RFS includes PHD \cite{Mahler2003firstMoment, Vo2005SMCPHD, Punithakumar2008immPHD, Vo2009immPHD}, CPHD \cite{Vo2007CPHD, Lundquist2013cphdExtendedTarget, Lundgren2013cphdSpawn}, MB \cite{Chai2020tbdMB} filters, etc.
PHD filter propagates the first-order statistical moment of the multi-target recursive Bayes nonlinear filter, which approximates the number of targets to the Poisson distribution.
It is computationally efficient and omits data associations.

In existing TBD-PHD filters, the updater is a direct extension from the standard point measurement model to the TBD model without rigorous Bayes derivation.
Although lots of research of TBD-PHD filter applied to radar systems have been published and applied to weak target tracking \cite{Punithakumar2005earliesBKPHD, BKPHD1, BKPHD4, BKPHD6, BKPHD9, BKPHD11, BKPHD12}, infrared maneuvering dim target tracking \cite{tbdPHD2012InfraredManeuvering}, image tracking \cite{image2016Comments}, smoother designing \cite{tbdPHD2016smoother}, and MIMO radar systems \cite{BKPHD2012SP}.
Most of them are improvements of the earlies one \cite{Punithakumar2005earliesBKPHD}.
An important one of these publications is the research from Habtemariam et al. \cite{BKPHD2012SP}.
It should be emphasized that different from DBT strategies, TBD is a non-threshold method.
Yet the formulas and assumptions of DBT-PHD are not suitable for TBD strategies.
Furthermore, neither rigorous Bayes prediction and update derivation nor approximation theory of recursive Bayes nonlinear filter is demonstrated in the above studies.

In this paper, a novel TBD-PHD filter is proposed in this paper, which is the best Poisson approximation based on Kullback-Leibler divergence (KLD) divergence minimization criteria in an information-theoretic sense.
In this paper, all assumptions of DBT and TBD strategies are enumerated and compared.
Furthermore, an analytical and principled close-form of TBD-PHD filter is derived based on Bayes rules.
Besides, Poisson approximation is implicitly executed after DBT-PHD updating in Fig. \ref{fig_flowchart}.
Contrasted with DBT-PHD filter, due to the Poisson conjugate prior for TBD-PHD filter which is proven in this paper and enables easier and more efficient computation, the approximation of the cardinality distribution is not required.
Additionally, the capping operation after updating is designed.
Eventually, the SMC implementation is devised for the radar system.

\begin{figure}[hbtp]
\centering
	\begin{minipage}[t]{0.445\textwidth}
		\centering
		\includegraphics[width=1.0\textwidth]{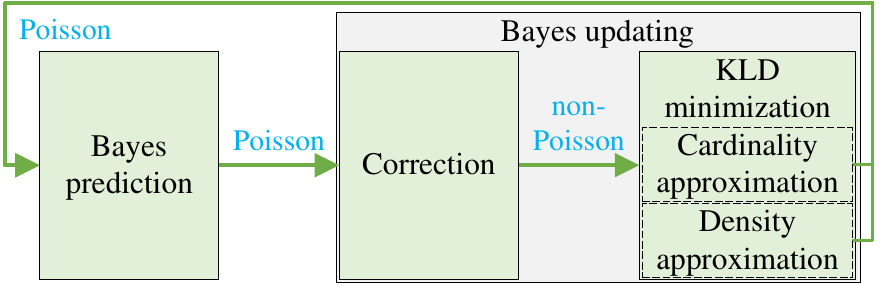}\\
		\centering{(a)}\\
	\end{minipage}%
	
	\begin{minipage}[t]{0.445\textwidth}
		\centering
		\includegraphics[width=1.0\textwidth]{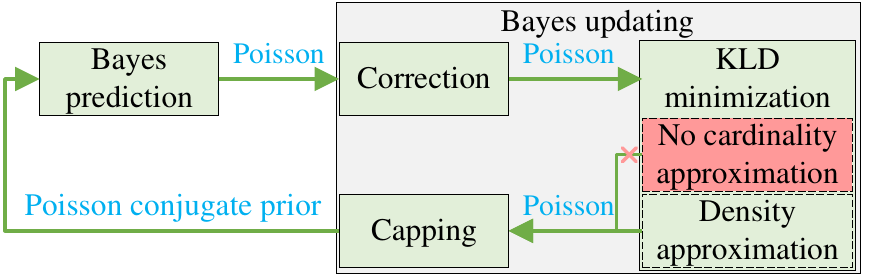}\\
		\centering{(b)}\\
	\end{minipage}

	\centering{
\caption{Flowchart of PHD filtering based KLD minimization criteria with 2 different strategies: (a) PHD filtering based on DBT strategies, (b) PHD filtering based on TBD strategies.} \label{fig_flowchart}
}
\end{figure}

\section{Models and notations}

\subsection{Dynamic model}

Assume that the target motions are independent of each other. The dynamic model of a single target is
\begin{equation}
{{x}_k} = f\left( {{{x}_{k - 1}},{v_k}} \right),
\end{equation}
where the state is ${{x}_k} = \left[ {p_k^x,\dot p_k^x,p_k^y,\dot p_k^y} \right]^{\rm{T}} \in {\mathbb{R}^{4}}$, and process noise is ${v_k}$.
Here, transpose is ${\left[  \cdot  \right]^{\rm{T}}}$.
Positions are ${p_k^x}, {p_k^y}$.
Velocities are ${\dot p_k^x}, {\dot p_k^y}$.
The multi-target state set ${X_k} = \left\{ {{x_{k,1}}, \cdots ,{x_{k,n}}} \right\} \in {\cal F}\left( {\mathbb{R}^{4}} \right)$ is an instance of RFS $\Xi_k$, where ${\cal F}\left( {{\mathbb{R}^{{4}}}} \right)$ denotes multiple-target state space.

\subsection{Observation model}

\subsubsection{Standard point measurement model for DBT strategies}

The PHD filter based on DBT strategies presumes a standard point measurement model.
Point measurement ${z_k} = {\left[ {{r_k},{\theta _k}} \right]^{\rm{T}}} \in {\mathbb{R}^{{2}}}$ consists of range-bearing as shown in Fig. \ref{figSubPointMea}.
The measurement set ${Z_k} = \left\{ {{z_{k,1}}, \cdots ,{z_{k,m}}} \right\} \in {\cal F}\left( {\mathbb{R}^{2}} \right)$ is an instance of RFS $\Sigma$, where
${\cal F}\left( {{\mathbb{R}^{{2}}}} \right)$ denotes the whole observation space for DBT strategies.
The measurement RFS  collected by the sensor is $\Sigma  = \Upsilon \left( X \right) \cup C$, where $\Upsilon(X)$ is generated by the targets, $C$ is Poisson clutter RFS.
The average number of clutters per unit volume is $\lambda_c$, the spatial distribution is $c(z)$, the volume of field-of-view (FOV) is $V$, and the intensity of Poisson clutter is $\kappa(z)=V \lambda_c  c(z)$.

Assumptions:
\begin{enumerate}[U1.]
  \item\label{item_U1} No target occupies more than one pixel.
  \item\label{item_U2} Each pixel is occupied by no more than one target.
  \item\label{item_U3} The measurements generated by targets and clutters are independent of each other.
  \item\label{item_U4} The target generates a point measurement with detection probability $p_D(x)$, otherwise the detection is missing with probability $1-p_D(x)$.
\end{enumerate}

\subsubsection{Amplitude echo model for TBD strategies}

Amplitudes of echoes are stored in an array with a fixed number $m$ based on TBD strategies, i.e., ${Z_k} = \left[ {{z_{k,1}}, \cdots, {z_{k,m}}} \right]$, which indicates fixed $m$ pixels.
Here, each ${z_k} = \left[ {{r_k},{\theta _k},{A_k}} \right]^{\rm{T}} \in {\mathbb{R} ^{{3}}}$ is composed of range-bearing-amplitude.
Therefore, the measurement set is not modeled as RFS.
In other words, there is no measurement RFS $\Upsilon \left( X \right)$ and clutter RFS $C$.
TBD-based updater directly uses the amplitude of the raw echo data as exhibited in Fig. \ref{figSubRawData} instead of a set of point measurements by passing a defined threshold.

\begin{figure}[!tbp]
    \center

    \subfloat[b][Point measurements\label{figSubPointMea}]{  
		\includegraphics[width=0.23\textwidth]{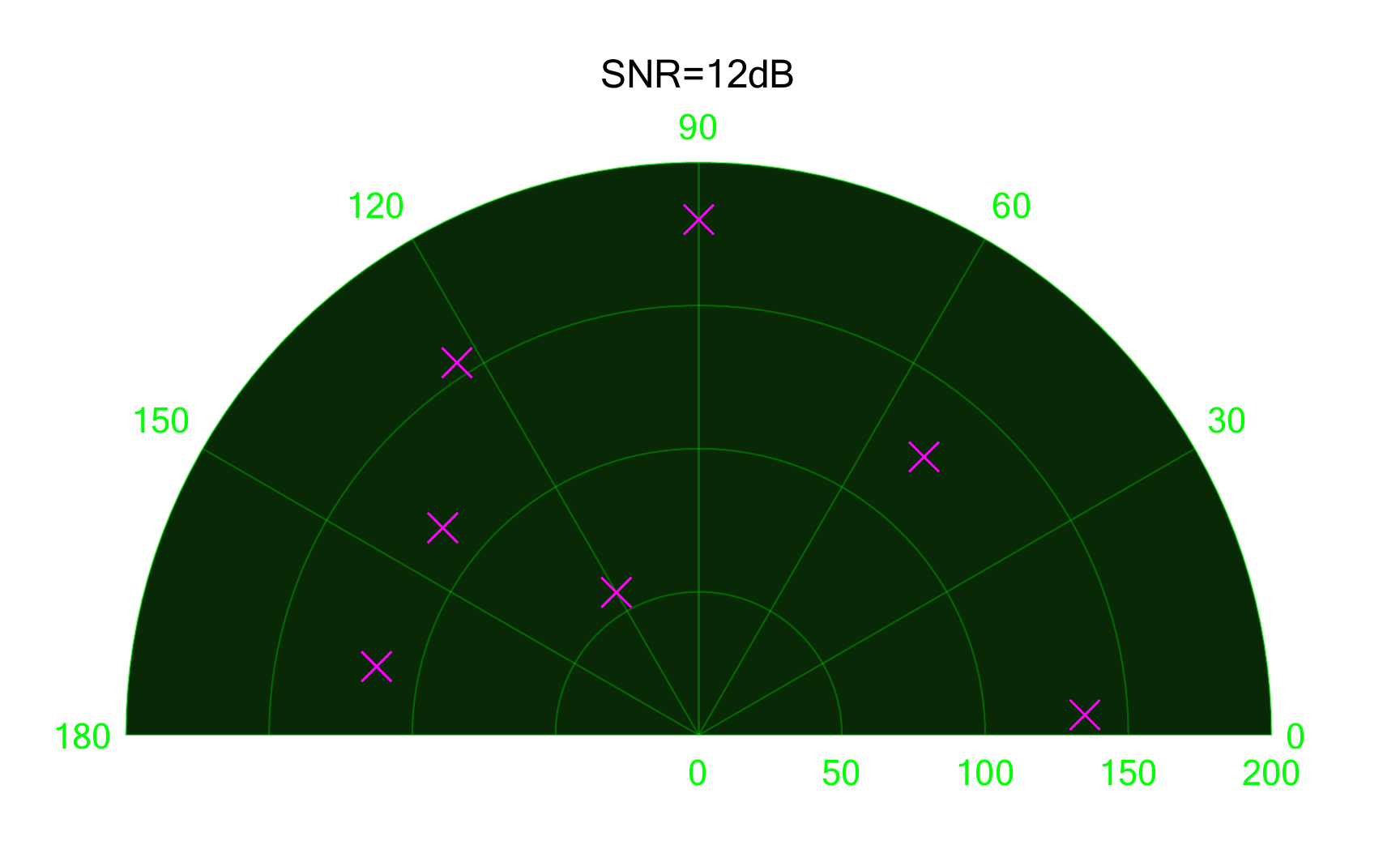}}
    \subfloat[b][Amplitudes of echoes\label{figSubRawData}]{
		\includegraphics[width=0.23\textwidth]{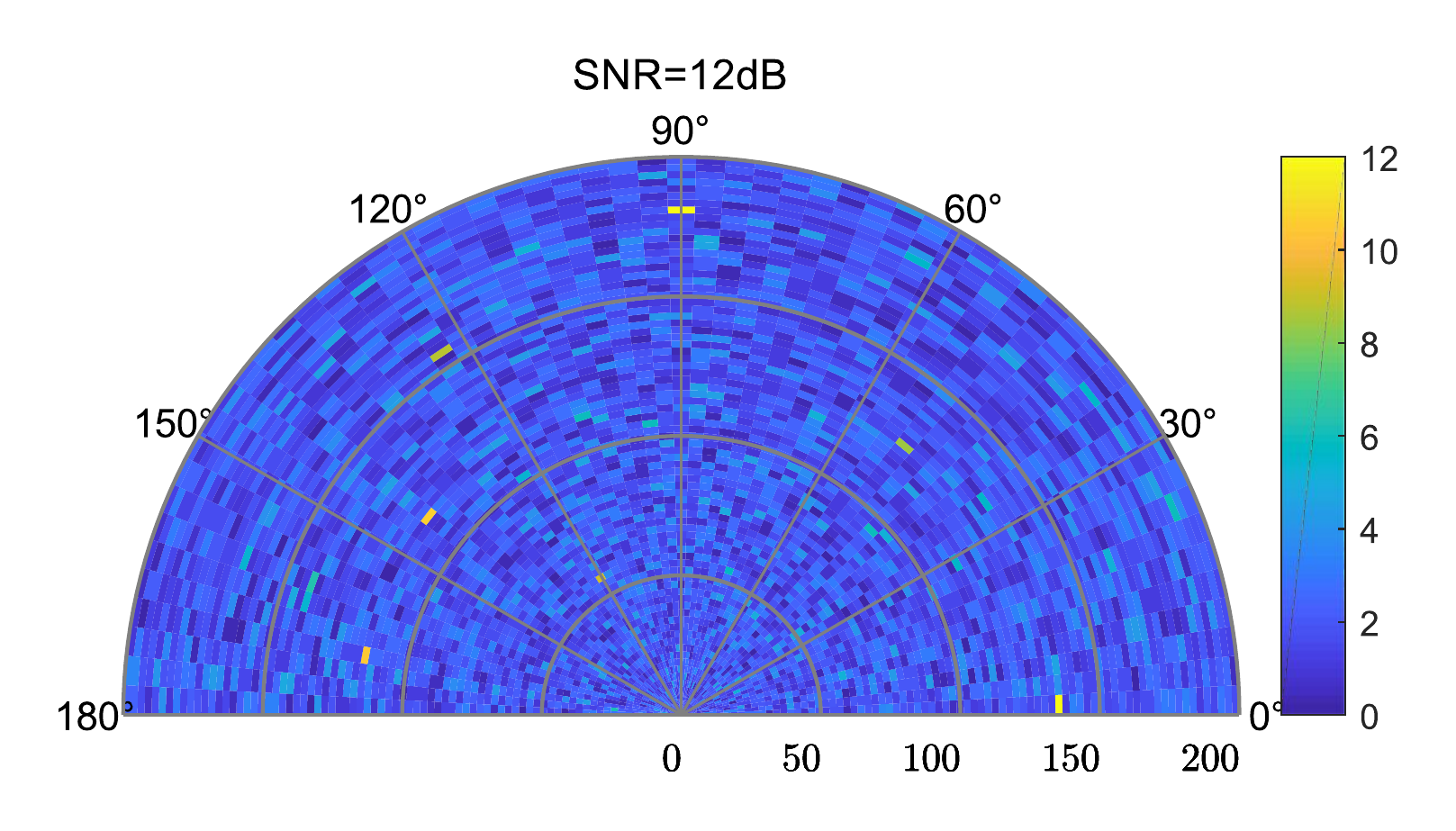}}
    \caption{The measurement models based DBT and TBD strategies when SNR=12dB.
    (a) Point Measurements are obtained by a constant false alarm rate (CFAR) detector with $p_{\rm{FA}}=10^{-4}$.
    (b) The amplitude of raw measurement data obeys Rayleigh distribution with the parameter $\sigma_n^2=1.5$. \label{figMeasurementModel}}

\end{figure}

Assumptions:
\begin{enumerate}[P1.]
  \item\label{item_P1} The measurements generated by targets and noise are independent mutually.
  \item\label{item_P2} The targets are non-spread point targets, and each target occupies only one pixel.
  \item\label{item_P3} The set of pixel $i$ illuminated by the target $x$ is denoted as $T(x)$.
  The spatial position is denoted by ${[ {r_k^{\left( i \right)},\theta _k^{\left( i \right)}} ]^{\rm{T}}}$, the amplitude is denoted by $A_k^{(i)}$, and they are conditionally independent of the target $x$.
  \item\label{item_P4} The targets are far apart, i.e., $x\neq x'$ leading to $T\left( x \right) \cap T\left( {{x'}} \right) = \emptyset $.
  \item\label{item_P5} Multi-target likelihood $g\left( {{Z_{k}}|X} \right)$ is separable.
\end{enumerate}

\section{Problem statement}

Not considering spawning targets,
the closed-form solution of the predictor and updater of PHD filter based on DBT strategies is \cite{Vo2006GMPHD}
\begin{equation}
{v_{k\mid k - 1}}(x) = \int {{p_{S,k}}(\xi ){f_{k\mid k - 1}}(x|\xi ){v_{k - 1}}(\xi )d\xi }  + {\gamma _k}(x),
\end{equation}
\begin{equation}
\begin{aligned}
{v_k}(x) &= \left[ {1 - {p_{D,k}}(x)} \right]{v_{k|k - 1}}(x)\\
& + \sum\limits_{{z_k} \in {Z_k}} {\frac{{{p_{D,k}}(x)\ell ({z_k}|x){v_{k|k - 1}}(x)}}{{\kappa ({z_k}) + \int {{p_{D,k}}(\zeta )\ell ({z_k}|\zeta ){v_{k|k - 1}}(\zeta )d\zeta } }}}.
\end{aligned}
\end{equation}

In summary, different from DBT strategies, the TBD-based filter shouldn't define detection probability $p_D$, missed detection $1-p_D$, false alarm process $C$, Poisson clutter mean $\lambda_c$, clutter spatial density $c(z)$, clutter intensity $\kappa(z)$.
In other words, they shouldn't appear in the derivation and result of the TBD-PHD filter.

\begin{remark}
On the contrary, simply setting $p_D\left(x\right)=1$ while retaining $\lambda_c, {c\left(z\right)}, \kappa(z)$ in the updater of existing TBD-PHD filters \cite{Punithakumar2005earliesBKPHD, BKPHD1, BKPHD4, BKPHD6, BKPHD9, BKPHD11, BKPHD12, tbdPHD2012InfraredManeuvering, image2016Comments, tbdPHD2016smoother, BKPHD2012SP}, i.e., \cite[equation (25)]{BKPHD2012SP}
\begin{equation}\label{equ_wrongTBDphdUpdater}
{v_k}(x) = \sum\limits_{{z_k} \in {Z_k}} {\frac{{g({z_k}|x){v_{k|k - 1}}(x)}}{{{\kappa}(z_k) + \int {g({z_k}|\xi ){v_{k|k - 1}}(\xi )d\xi } }}}
\end{equation}
violates the assumptions above.
In other words, measurement update formulas for DBT-PHD filter are just mechanically and unreasonably applied to TBD-PHD filters.
For instance, the presence of the clutter term ${\kappa}(z_k)$ in (\ref{equ_wrongTBDphdUpdater}) is unreasonable.
Therefore, a principled PHD filter with TBD strategies is eager to be developed based on rigorous Bayes rules.
\end{remark}

\section{TBD-PHD filter derivation based on KLD minimization}

\subsection{Apply KLD minimization to PHD filter}\label{Sec_Appendix_KLD}

PHD filter implicitly performs KLD minimization in the updating step to obtain the best Poisson approximation of the multi-target density, recursively.
\'Angel and Vo derived the DBT-PHD/CPHD filter based on KLD minimization \cite{Angel2015KLD}.
The approximate intensity function ${v_\pi}\left( x \right)$ can be completely characterized by the single target density $\pi\left( x \right)$ and cardinality distribution ${\rho _\pi}\left( n \right)$.

Assuming the real joint density of individual targets is $q\left(  X  \right)$, it can be any continuous distribution.
The real cardinality distribution is ${\rho _q}\left(  n  \right)$, which can be any probability mass function (p.m.f.).
The density of the Poisson approximation is $\pi \left(  X  \right)$.
The approximate cardinality distribution is ${\rho _\pi }\left(  n  \right)$.
KLD is
\begin{subequations}
\begin{align}
& {\rm{KLD}}\left( {q||\pi } \right) = \int {q\left( X \right)\log \frac{{q\left( X \right)}}{{\pi \left( X \right)}}\delta X} \\
&= \sum\limits_{n = 0}^\infty  {\frac{1}{{n!}}\int {q\left( {\left\{ {{x_1}, \cdots ,{x_n}} \right\}} \right)}  \log \frac{{q\left( {\left\{ {{x_1}, \cdots ,{x_n}} \right\}} \right)}}{{\pi \left( {\left\{ {{x_1}, \cdots ,{x_n}} \right\}} \right)}}d{x_{1:n}}} \\
&=\sum\limits_{n = 0}^\infty  {{\rho _q}\left( n \right) \int { {q\left( {{x_1}, \cdots ,{x_n}} \right)}  {\log \frac{{q\left( {{x_1}, \cdots ,{x_n}} \right)}}{{\prod_{j = 1}^n {\pi \left( {{x_j}} \right)} }}} dx} } \notag \\
 &+ \sum\limits_{n = 0}^\infty  {{\rho _q}\left( n \right)\log \frac{{{\rho _q}\left( n \right)}}{{{\rho _\pi }\left( n \right)}}}, \label{equ_kld}
\end{align}
\end{subequations}
where the first term $q\left( {{x_1}, \cdots ,{x_n}} \right)$ in (\ref{equ_kld}) is real joint p.d.f. of multi-target, whereas marginal p.d.f. of $q(\cdot)$ is not i.i.d. \cite[Section 11.3.3]{Mahler2007FISST}.
In contrast, $\prod_{j = 1}^n {\pi \left( {{x_j}} \right)} $ is approximate p.d.f. of $q\left( {{x_1}, \cdots ,{x_n}} \right)$, meanwhile individuals of $\pi(\cdot)$ is i.i.d.
When $q\left( {{x_1}, \cdots ,{x_n}} \right) = \prod_{j = 1}^n {\pi \left( {{x_j}} \right)} $, the first term of KLD in (\ref{equ_kld}) minimizes so that we get the best approximation density $\pi(x)$ over the single target space $\mathbb{R}^{4}$.
When ${\rho _q}\left( n \right) = {\rho _\pi }\left( n \right)$, second term of KLD in (\ref{equ_kld}) minimizes.
Provided ${q\left( {{x_1}, \cdots ,{x_n}} \right)},{\rho _q}\left( n \right)$, the best Poisson approximation after KLD minimization is
${\rho _\pi } = {{\exp \left( { - {\lambda _\pi }} \right) \cdot \lambda _\pi ^n} \mathord{\left/
 {\vphantom {{\exp \left( { - {\lambda _\pi }} \right) \cdot \lambda _\pi ^n} {n!}}} \right.
 \kern-\nulldelimiterspace} {n!}}$ with parameter
 ${\lambda _\pi } = \sum_{n = 0}^\infty  {n \cdot {\rho _q}\left( n \right)} $.
In summary, the best Poisson approximation of density $\pi(x)$ and cardinality distribution $\rho_\pi(n)$ is obtained based on KLD minimization.
Hereby, we get intensity function ${v_\pi }\left( x \right)= \lambda_{\pi} \cdot {\pi(x)}$.

\subsection{TBD-PHD filter predictor}

The posterior multi-target RFS is ${\Xi _k}$, the cardinality distribution is ${\rho _{\left| {{\Xi _k}} \right|}}$, the posterior intensity function is ${v_k}$, and the posterior multi-target density is $\pi \left(  \cdot  \right)\mathop  = \limits^{{\rm{abbr}}{\rm{.}}} {\pi _k}\left( { \cdot |{Z_k}} \right)$.
The predicted RFS is ${\Xi _{k + 1|k}} = {\Gamma _{k + 1}} \cup {{\cal S}_{k + 1|k}}$, the cardinality distribution is ${\rho _{\left| {{\Gamma _{k + 1}}} \right|}}$, and the intensity function is ${\gamma _{k + 1}}$.
RFS of survival target is ${{\cal S}_{k + 1|k}} = x \cup {\mathord{\buildrel{\lower3pt\hbox{$\scriptscriptstyle\frown$}}
\over {\cal S}} _{k + 1|k}}$, cardinality distribution is ${\rho _{\left| {{{\cal S}_{k + 1|k}}} \right|}}$, and intensity function is ${v_{{\cal S},k + 1|k}}$.
The survival probability is ${p_S}$, and the Markov transition density of single target is ${f_{k + 1|k}}\left( { \cdot |\zeta } \right)$.
It is known that the cardinality distribution ${\rho _{\left| {{\Xi _k}} \right|}}$ is a Poisson distribution with parameter $\left\langle {{v_k},1} \right\rangle $.
The spawn target is not considered in the prediction.

\begin{proposition}
The predicted cardinality distribution of TBD-PHD filter is
\begin{equation}
{\rho _{\left| {{\Xi _{k + 1|k}}} \right|}}\left( n \right) = {\rho _{\left| {{\Gamma _{k + 1}}} \right|}}\left( n \right)*{\rho _{\left| {{{\cal S}_{k + 1|k}}} \right|}}\left( n \right),
\end{equation}
where $*$ denotes the convolution. Provided that ${\rho _{\left| {{\Gamma _{k + 1}}} \right|}}$ and ${\rho _{\left| {{{\cal S}_{k + 1|k}}} \right|}}$ are Poisson distributions with parameters $\left\langle {{\gamma _{k + 1}},1} \right\rangle $ and $\left\langle {{p_S},{v_k}} \right\rangle $ already, the prediction step doesn't need the Poisson approximation.
Because the cardinality distribution of the predictor is already Poisson distribution with parameter $\left\langle {{\gamma _{k + 1}},1} \right\rangle  + \left\langle {{p_S},{v_k}} \right\rangle $.
\end{proposition}

\begin{proposition}
The predicted intensity function of TBD-PHD filter is
\begin{equation}
\begin{aligned}
{v_{k + 1|k}}\left( x \right) &= {\gamma _{k + 1}}\left( x \right) {\rm{ + }}\\
&\int {{p_S}\left( \zeta  \right) \cdot {f_{k + 1|k}}\left( {x|\zeta } \right) \cdot {v_k}\left( \zeta  \right)d\zeta }.
\end{aligned}
\end{equation}
\end{proposition}

\subsection{TBD-PHD filter updater}
Similar to the DBT-PHD filter derivation based on KLD minimization, FISST and set integrals are essential tools.
These derivations strictly obey the Bayes rule.

FOV consists of $m$ pixels.
Assuming that the pixels illuminated by the target $x\in X$ do not overlap.
\begin{defination}
The likelihood ratio (LR) of target $x$ is
\begin{equation}
\mathbb{L} \left( {{Z_{k }}|x} \right) = \prod\limits_{i \in T\left( x \right)} {\frac{{g\left( {A_{k }^{\left( i \right)}|x,{{\cal H}_1}} \right)}}{{g\left( {A_{k }^{\left( i \right)}|{{\cal H}_0}} \right)}}},
\end{equation}
where the likelihood function of pixel $i$ illuminated by $x$ is $g\left( {A_{k + 1}^{\left( i \right)}|x,{{\cal H}_1}} \right)$, but that of pixel $i$ without $x$ is $ g\left( {A_{k + 1}^{\left( i \right)}|{{\cal H}_0}} \right)$.
\end{defination}
\begin{defination}
According to assumption {P\ref{item_P3}} and {P\ref{item_P5}}, the separable likelihood function is
\begin{subequations}\label{equ_separableLikelihood}
\begin{align}
g\left( {{Z_{k}}|X} \right) &= \left( {\prod\limits_{x \in X} {\prod\limits_{i \in T\left( x \right)} {g\left( {A_{k }^{\left( i \right)}|x,{{\cal H}_1}} \right)} } } \right) \label{equ_signalLikelihood} \\
& \times \left( {\prod\limits_{i \notin \bigcup\limits_{x \in X} {T\left( x \right)} } {g\left( {A_{k }^{\left( i \right)}|{{\cal H}_0}} \right)} } \right) \label{equ_noiseLikelihood} \\
& = g\left( {{Z_{k }}|{{\cal H}_0}} \right)\prod\limits_{x \in X} \mathbb{L}{\left( {{Z_{k }}|x} \right)}. \label{equ_noTargetLikelihood}
\end{align}
\end{subequations}
where ${G_0}\mathop  = \limits^{{\rm{abbr}}{\rm{.}}}  g\left( {{Z_{k }}|{{\cal H}_0}} \right) = \prod_{i = 1}^m {g\left( {A_{k }^{\left( i \right)}|{{\cal H}_0}} \right)} $ in (\ref{equ_noTargetLikelihood}) is the likelihood function of noise abbreviated as $G_0$.
\end{defination}
\begin{proposition}\label{proposi_updated_PHD}
The updated intensity function of TBD-PHD filter is
\begin{equation}\label{updated_PHD}
{v_{k + 1}}\left( x \right) = \mathbb{L}\left( {{Z_{k + 1}}|x} \right) \cdot {v_{k + 1|k}}\left( x \right).
\end{equation}
\end{proposition}
\begin{proof}\label{proof_updated_PHD}
Let the state set ${X_k} = \left\{ {x,{w_1}, \cdots ,{w_{n - 1}}} \right\} \equiv x \cup W$.
The abbreviation for density and Poisson parameters is $\pi \mathop  = \limits^{{\rm{abbr}}{\rm{.}}} {\pi _{k + 1|k}}$, $\lambda \mathop  = \limits^{{\rm{abbr}}{\rm{.}}} {\lambda _{k + 1|k}}$.
The intensity function based on the Bayes equation is defined as
\begin{equation}\label{equ_bayesPHD}
{v_{k + 1}}\left( x \right) = \frac{{\int {g\left( {{Z_{k + 1}}|\left\{ x \right\} \cup W} \right)    \pi \left( {\left\{ x \right\} \cup W|{Z_k}} \right)\delta W} }}{{p\left( {{Z_{k + 1}}} \right)}}
\end{equation}
where $p\left( {{Z_{k + 1}}} \right) = \int {g\left( {{Z_{k + 1}}|\mathord{\buildrel{\lower3pt\hbox{$\scriptscriptstyle\smile$}}
\over X} } \right)  \pi \left( {\mathord{\buildrel{\lower3pt\hbox{$\scriptscriptstyle\smile$}}
\over X} |{Z_k}} \right)\delta \mathord{\buildrel{\lower3pt\hbox{$\scriptscriptstyle\smile$}}
\over X} } $.
Owing to
\begin{equation}\label{equ_setPDF2jointPDF}
\pi \left( {\left\{ {{x_1}, \cdots ,{x_n}} \right\}} \right) = {\rho _\pi }\left( n \right) \cdot n! \cdot \prod\limits_{i = 1}^n {\pi \left( {{x_i}} \right)},
\end{equation}
Then, the numerator of (\ref{equ_bayesPHD}) is obtained by (\ref{equ_separableLikelihood}) and (\ref{equ_setPDF2jointPDF}) as
\begin{equation}
\begin{aligned}
&  \frac{{{\exp({ - \lambda })}}}{{n!}}  {G_0}  \left[ {\prod\limits_{\scriptstyle{i} \in T\left( x \right)\hfill\atop
\scriptstyle{x} \in X\hfill} {\frac{{g\left( {z_k^{\left( i \right)}|x,{{\cal H}_1}} \right)}}{{g\left( {z_k^{\left( i \right)}|{{\cal H}_0}} \right)}}}   \lambda \pi \left( x \right)} \right] \times \\
& \sum\limits_{n = 1}^\infty  {\frac{{n!}}{{\left( {n - 1} \right)!}}  {{\left[ {\int {\prod\limits_{\scriptstyle{i} \in T\left( w \right)\hfill\atop
\scriptstyle{w} \in X\hfill} {\frac{{g\left( {z_k^{\left( i \right)}|w,{{\cal H}_1}} \right)}}{{g\left( {z_k^{\left( i \right)}|{{\cal H}_0}} \right)}}}   \lambda \pi \left( w \right)dw} } \right]}^{n - 1}}}.
\end{aligned}
\end{equation}
Besides, ${p\left( {{Z_{k + 1}}} \right)}$ from (\ref{equ_bayesPHD}) is
\begin{equation}
\begin{aligned}
& p\left( {{Z_{k + 1}}} \right) = \frac{{{\exp({ - \lambda })}}}{{n!}}{G_0} \\
& \times \sum\limits_{\mathord{\buildrel{\lower3pt\hbox{$\scriptscriptstyle\smile$}}
\over n}  = 0}^\infty  {\frac{1}{{\mathord{\buildrel{\lower3pt\hbox{$\scriptscriptstyle\smile$}}
\over n} !}}{{\left[ {\int {\prod\limits_{\scriptstyle{j} \in T\left( {\mathord{\buildrel{\lower3pt\hbox{$\scriptscriptstyle\smile$}}
\over x} } \right)\hfill\atop
\scriptstyle\mathord{\buildrel{\lower3pt\hbox{$\scriptscriptstyle\smile$}}
\over x}  \in X\hfill} {\frac{{g\left( {z_k^{\left( j \right)}|\mathord{\buildrel{\lower3pt\hbox{$\scriptscriptstyle\smile$}}
\over x} ,{{\cal H}_1}} \right)}}{{g\left( {z_k^{\left( j \right)}|{{\cal H}_0}} \right)}}}   \lambda \pi \left( {\mathord{\buildrel{\lower3pt\hbox{$\scriptscriptstyle\smile$}}
\over x} } \right)d\mathord{\buildrel{\lower3pt\hbox{$\scriptscriptstyle\smile$}}
\over x} } } \right]}^{\mathord{\buildrel{\lower3pt\hbox{$\scriptscriptstyle\smile$}}
\over n} }}}.  \label{equ_denomCombine}
\end{aligned}
\end{equation}
Hence, the intensity function of (\ref{equ_bayesPHD}) satisfies
\begin{subequations}
\begin{align}
&{v_{k + 1}}\left( x \right) = \left[\mathbb{L} {\left( {{Z_{k + 1}}|x} \right) \cdot \lambda \pi \left( x \right)} \right] \label{equ_lambdaPi2PHD} \\
& \times \frac{{\sum\limits_{n = 1}^\infty  {{{{{\left[ {\int \mathbb{L}{\left( {{Z_{k + 1}}|w} \right) \cdot \lambda \pi \left( w \right)dw} } \right]}^{n - 1}}} \mathord{\left/
 {\vphantom {{{{\left[ {\int {\left( {{Z_{k + 1}}|w} \right) \cdot \lambda \pi \left( w \right)dw} } \right]}^{n - 1}}} {\left( {n - 1} \right)!}}} \right.
 \kern-\nulldelimiterspace} {\left( {n - 1} \right)!}}} }}{{\sum\limits_{\mathord{\buildrel{\lower3pt\hbox{$\scriptscriptstyle\smile$}}
\over n}  = 0}^\infty  {\frac{1}{{\mathord{\buildrel{\lower3pt\hbox{$\scriptscriptstyle\smile$}}
\over n} !}}{{\left[ {\int \mathbb{L}{\left( {{Z_{k + 1}}|y} \right) \cdot \lambda \pi \left( y \right)dy} } \right]}^{\mathord{\buildrel{\lower3pt\hbox{$\scriptscriptstyle\smile$}}
\over n} }}} }}. \label{equ_denomEliminate}
\end{align}
\end{subequations}
Here, $\lambda \pi \left( x \right) = {v_{k + 1|k}}\left( x \right)$ in (\ref{equ_lambdaPi2PHD}), and (\ref{equ_denomEliminate}) is equal to $1$.
Therefore, the proposition \ref{proposi_updated_PHD} is proven.
\end{proof}

\begin{proposition}\label{proposi_updated_cd}
Supposing that the prior cardinality distribution is Poisson, the updated cardinality distribution of TBD-PHD filter is
\begin{equation}\label{updated_cd}
\begin{aligned}
{\rho _{\left| {{\Xi _{k + 1}}} \right|}}\left( n \right) &= \exp \left( { -  \left\langle {\mathbb{L}\left( {{Z_{k + 1}}| \cdot } \right),{v_{k + 1|k}\left(\cdot\right)}} \right\rangle } \right) \\
&\times \frac{{ {{\left\langle {\mathbb{L} \left( {{Z_{k + 1}}| \cdot } \right),{v_{k + 1|k}\left(\cdot\right)}} \right\rangle }^n}}}{{n!}}.
\end{aligned}
\end{equation}
\end{proposition}

Therefore, the TBD-PHD posterior cardinality distribution is the Poisson distribution with parameter ${\lambda _{k + 1}} = \left\langle {\mathbb{L}\left( {{Z_{k + 1}}| \cdot } \right),{v_{k + 1|k}\left(\cdot\right)}} \right\rangle $.
In addition, the Poisson prior cardinality distribution of TBD-PHD is conjugated to any likelihood function.
\begin{proof}\label{proof_updated_cd}
For TBD-PHD filter, the cardinality distribution based on the Bayes equation is
\begin{subequations}
\begin{align}
& {\rho _{\left| {{\Xi _{k + 1}}} \right|}}\left( n \right) = \frac{1}{{n!}}\int {{\pi _{k + 1}}\left( {\left\{ {{x_1}, \cdots ,{x_n}} \right\}} \right)d\left( {{x_1} \cdots {x_n}} \right)} \\
& =\frac{{\frac{1}{{n!}}\int {g\left( {{Z_{k + 1|}}|{x_1}, \cdots ,{x_n}} \right) \cdot n! \cdot \rho \left( n \right)\prod\limits_{i = 1}^n {\pi \left( {{x_i}} \right)} d{x_{1:n}}} }}{
{\left[ \begin{array}{l}
\rho \left( 0 \right) \cdot g\left( {{Z_{k + 1}}|{{\cal H}_0}} \right) \cdot f\left( \emptyset  \right) + \\
\sum\limits_{\mathord{\buildrel{\lower3pt\hbox{$\scriptscriptstyle\smile$}}
\over n}  = 1}^\infty  {\frac{{\rho \left( {\mathord{\buildrel{\lower3pt\hbox{$\scriptscriptstyle\smile$}}
\over n} } \right)}}{{\mathord{\buildrel{\lower3pt\hbox{$\scriptscriptstyle\smile$}}
\over n} !}}\int {g\left( {{Z_{k + 1}}|{x_1}, \cdots ,{x_{\mathord{\buildrel{\lower3pt\hbox{$\scriptscriptstyle\smile$}}
\over n} }}} \right)  \mathord{\buildrel{\lower3pt\hbox{$\scriptscriptstyle\smile$}}
\over n} !  \prod\limits_{j = 1}^{\mathord{\buildrel{\lower3pt\hbox{$\scriptscriptstyle\smile$}}
\over n} } {\pi \left( {{x_j}} \right)} d{x_j}} }
\end{array} \right]}
}.
\end{align}
\end{subequations}
where $f\left( \emptyset  \right) = 1$.
By (\ref{equ_separableLikelihood}) and (\ref{equ_setPDF2jointPDF}), we get
\begin{equation}
\begin{aligned}
{\rho _{\left| {{\Xi _{k + 1}}} \right|}}\left( n \right)=
& \frac{{{{{{\left[ {\int\mathbb{L} {\left( {{Z_{k + 1}}|x} \right) \cdot {v_{k + 1|k}}\left( x \right)dx} } \right]}^n}} \mathord{\left/
 {\vphantom {{{{\left[ {\int\mathbb{L} {\left( {{Z_{k + 1}}|x} \right) \cdot {v_{k + 1|k}}\left( x \right)dx} } \right]}^n}} {n!}}} \right.
 \kern-\nulldelimiterspace} {n!}}}}{{\sum\limits_{\mathord{\buildrel{\lower3pt\hbox{$\scriptscriptstyle\smile$}}
\over n}  = 0}^\infty  {\frac{{{{\left[ {\int\mathbb{L} {\left( {{Z_{k + 1}}|x} \right) \cdot {v_{k + 1|k}}\left( x \right)dx} } \right]}^{\mathord{\buildrel{\lower3pt\hbox{$\scriptscriptstyle\smile$}}
\over n} }}}}{{\mathord{\buildrel{\lower3pt\hbox{$\scriptscriptstyle\smile$}}
\over n} !}}} }}. \label{equ_Maclaurin}
\end{aligned}
\end{equation}
According to the Maclaurin series $\sum\nolimits_{n = 0}^{ + \infty } {{{{\lambda ^n}} \mathord{\left/
 {\vphantom {{{\lambda ^n}} {n!}}} \right.
 \kern-\nulldelimiterspace} {n!}}} = \exp \left( \lambda  \right)$, the denominator of (\ref{equ_Maclaurin}) is equal to $\exp \left( {\left\langle \mathbb{L}{\left( {{Z_{k + 1}}| \cdot } \right),{v_{k + 1|k}}} \right\rangle } \right)$.
Therefore, proposition \ref{proposi_updated_cd} is proven.

\end{proof}

\begin{remark}
Unlike (\ref{equ_wrongTBDphdUpdater}), the newly proposed PHD filter in (\ref{updated_PHD}), (\ref{updated_cd}) no longer includes parameters ${p_D}\left( x \right),{\lambda _c},c\left( z \right),\kappa \left( z \right)$, which exist only in DBT-based filters and are coupled to each other.
If set ${p_D} = 1$ as (\ref{equ_wrongTBDphdUpdater}), then the Neyman-Pearson (NP) criterion fails.
Logically, this will result in the false alarm probability $p_{\rm{FA}}=1$, which indicates that all measurements are from clutters.
Equivalently, no measurement comes from targets.
In practice, engineers can only try to tune the parameter of $\kappa \left( z \right)$ by experience.
This is absurd.
\end{remark}

\subsection{Capping}

The number of updated targets is ${N_{k + 1}} = \int {{v_{k + 1}}\left( x \right)dx} $.
A large signal-to-noise ratio (SNR) scenario will lead to an excessively large LR of the (\ref{updated_PHD}), which results in the divergence of the number estimation of targets $N_{k+1}$.
Counterintuitively, the number of targets increases with the increase of LR (or SNR) as shown in (\ref{updated_PHD}).
Although TBD-PHD does not require additional Poisson approximation, capping operation is still required.
We have two capping schemes.
Scheme 1: Cap the intensity of each target component, so that its integral is no more than 1.
Scheme 2: Cap the intensity of each pixel.
In this paper, we select scheme 1.

\begin{proposition}\label{proposi_updated_capping}
Assume that the prediction has $N$ mutually isolated target components
\begin{equation}
{v_{k + 1|k}}\left( x \right) = \sum\limits_{i = 1}^N {{D_{k + 1|k,i}}\left( x \right)}.
\end{equation}
Capping is performed on each posterior updated component. The intensity function of each component is
\begin{equation}
{D_{k + 1,i}}\left( x \right) = \frac{{\mathbb{L}\left( {{Z_{k + 1}}|x} \right) \cdot {D_{k + 1|k,i}}\left( x \right)}}{{\max \left( {\int {\mathbb{L}\left( {{Z_{k + 1}}|x} \right) \cdot {D_{k + 1|k,i}}\left( x \right)dx} ,1} \right)}}.
\end{equation}
The cardinality distribution of the posterior updater is rewritten as ${\rho _{k + 1}}\left( n \right) = {{\exp \left( { - \tilde \lambda } \right) \cdot {{\tilde \lambda }^n}} \mathord{\left/
 {\vphantom {{\exp \left( { - \tilde \lambda } \right) \cdot {{\tilde \lambda }^n}} {n!}}} \right.
 \kern-\nulldelimiterspace} {n!}}$, where $\tilde \lambda  = \sum_{i = 1}^N {\max \left( {\left\langle {\mathbb{L}\left( {{Z_{k + 1}}| \cdot } \right),{D_{k + 1|k,i}\left(\cdot\right)}} \right\rangle ,1} \right)} $.
Note that the capping doesn't change the conjugate prior characteristics.
\end{proposition}

Propositions \ref{proposi_updated_PHD}, \ref{proposi_updated_cd} and \ref{proposi_updated_capping} are the main result of this paper.
Detailed proofs for Propositions \ref{proposi_updated_PHD}, \ref{proposi_updated_cd}, and explanations for Proposition \ref{proposi_updated_capping} will be provided in a future journal article.

\section{TBD-PHD Implementation}

\subsection{Prediction}

The multi-target intensity function is initialized as a set of particles, whose location is uniformly distributed within pixels and whose velocity is Gaussian.
It avoids particle clustering when state extraction. Besides, capping is easy to implement.
The intensity function of the $i$-th target $D_i(x)$ has $P_i$ particles $\left\{ {\omega _i^{\left( j \right)},{{x}}_i^{\left( j \right)}} \right\}_{j = 1}^{{P_i}}$, the weight and state of $j$-th particle are $\omega _i^{\left( j \right)}$ and ${{x}}_i^{\left( j \right)}$.
Assuming that the predicted PHD is composed of a $L_{k}$ survival and $J_{k+1}$ newborn particles, then
\begin{equation}
{v_{k + 1|k}}\left( x \right) = \sum\limits_{p = 1}^{{L_{k}} + {J_{k + 1}}} {\omega _{k + 1|k}^{\left( p \right)}\delta \left( {x - {{x}}_{k + 1|k}^{\left( p \right)}} \right)}.
\end{equation}

The spontaneous birth target component is generated from pixels, whose amplitude is greater than the birth threshold. Simultaneously the pixels aren't occupied by particles at the last time step.
Index sets of particles are $L_{k}^S= \left\{ 1, \cdots ,{L_{k}} \right\}$ and $L_{k+1}^B= \left\{ L_{k}\rm{+}1, \cdots, {J_{k+1}} \right\}$.
States of the survival and birth target are sampled from the proposal distribution
\begin{equation}
{x}_{k + 1|k}^{\left( p \right)} \sim \left\{ {\begin{array}{*{20}{l}}
{{q_{k + 1}}\left( { \cdot |{\xi _k^{(p)}},{Z_k}} \right)}&{p \in {L_{k}^S} }\\
{{p_{k + 1}}\left( { \cdot |{Z_k}} \right)}&{p \in {L_{k+1}^B} }
\end{array}} \right.
.
\end{equation}
Particle weight satisfies
\begin{equation}
\omega _{k{\rm{ + }}1|k}^{\left( p \right)} = \left\{ {\begin{array}{*{20}{ll}}
{\frac{{{p_S}\left( {\xi _k^{\left( p \right)}} \right) \cdot {f_{k{\rm{ + }}1|k}}\left( {{x}_{k{\rm{ + }}1|k}^{\left( p \right)}|\xi _k^{\left( p \right)}} \right)}}{{{q_{k{\rm{ + }}1}}\left( {{x}_{k{\rm{ + }}1|k}^{\left( p \right)}|\xi _k^{\left( p \right)},{Z_k}} \right)}}}&{p \in {L_{k}^S} }\\
{\frac{{{\gamma _{k{\rm{ + }}1}}\left( {{x}_{k{\rm{ + }}1}^{\left( p \right)}} \right)}}{{{p_{k{\rm{ + }}1}}\left( {{x}_{k{\rm{ + }}1|k}^{\left( p \right)}|{Z_k}} \right)}}}&{p \in {L_{k+1}^B} }
\end{array}} \right.
.
\end{equation}

\subsection{Updating}

The SMC implementation of updating and capping is
\begin{equation}
\omega _{k+1|k+1,i}^{\left( j \right)} = \frac{{\mathbb{L}\left( {{Z_{k + 1}}|{x}_i^{\left( j \right)}} \right) \cdot \omega _{k + 1|k,i}^{\left( j \right)}}}{{\max \left( {\sum\limits_{j = 1}^{{P_i}} {\mathbb{L}\left( {{Z_{k + 1}}|{x}_i^{\left( j \right)}} \right) \cdot \omega _{k + 1|k,i}^{\left( j \right)}} ,1} \right)}}.
\end{equation}

\subsection{Re-sampling and multi-target state estimation}

Re-sample the particle set $\left\{ {\omega _{k + 1|k + 1}^{\left( j \right)},x_{k + 1|k}^{\left( j \right)}} \right\}_{j = 1}^{{L_{k}} + {J_{k+1}}}$ to obtain $\left\{ {\omega _{k + 1}^{\left( j \right)},x_{k + 1}^{\left( j \right)}} \right\}_{j = 1}^{{L_{k+1}}}$, and extract the state of each target component ${\hat x_{k + 1,i}} = {{\sum\nolimits_{j = 1}^{{P_i}} {\omega _{k + 1,i}^{\left( j \right)}x_{k + 1,i}^{\left( j \right)}} } \mathord{\left/
 {\vphantom {{\sum\nolimits_{j = 1}^{{P_i}} {\omega _{k + 1,i}^{\left( j \right)}x_{k + 1,i}^{\left( j \right)}} } {\sum\nolimits_{\ell  = 1}^{{P_i}} {\omega _{k + 1,i}^{\left( \ell  \right)}} }}} \right.
 \kern-\nulldelimiterspace} {\sum\nolimits_{\ell  = 1}^{{P_i}} {\omega _{k + 1,i}^{\left( \ell  \right)}} }}$.
The estimated number of posterior targets is ${\hat N_{k\rm{+}1}} = \sum_{i = 1}^N {\sum_{j = 1}^{{P_i}} {\omega _{k\rm{+}1,i}^{\left( j \right)}} } $. Mahalanobis distance is used to merge adjacent components.

\section{Simulation}
After the background noise passes through the radar I/Q dual channel, the likelihood function $g\left( {A_k^{\left( i \right)}|{{\cal H}_0}} \right)$ of amplitude of noise from (\ref{equ_noiseLikelihood}) obeys the Rayleigh distribution with parameter $\sigma _n^2$ as
\begin{equation}\label{equ_raylRoise}
{g}\left( A \right) = \frac{A}{{\sigma _n^2}}\exp \left( {\frac{{{-A^2}}}{{2\sigma _n^2}}} \right),\quad A > 0.
\end{equation}
The likelihood function  $g\left( {A_k^{\left( i \right)}|x,{{\cal H}_1}} \right)$ of amplitude of signal from (\ref{equ_signalLikelihood}) obeys Rayleigh distribution with parameter $\sigma _n^2 + \sigma _s^2$ as
\begin{equation}\label{equ_raylSignalPlusRoise}
{g}\left( A \right) = \frac{A}{{\left( {\sigma _n^2 + \sigma _s^2} \right)}}\exp \left( {\frac{{{-A^2}}}{{2\left( {\sigma _n^2 + \sigma _s^2} \right)}}} \right),\  A > 0.
\end{equation}
where ${\rm{SNR}} = 10{\log _{10}}\left( {{{\sigma _s^2} \mathord{\left/
 {\vphantom {{\sigma _s^2} {\sigma _n^2}}} \right.
 \kern-\nulldelimiterspace} {\sigma _n^2}}} \right)$ (dB) is assumed to be known.

\begin{table}[!tbp]
  \centering
  \caption{Target information. Birth time is $t_b$, lasting time is $t_l$, birth weight is $\omega$.}\label{table_initTargets}
    \begin{tabular}{ccccc}
    \toprule
    \multicolumn{1}{l}{target} & state & \multicolumn{1}{l}{$t_b$} & \multicolumn{1}{l}{$t_l$} & \multicolumn{1}{l}{$\omega$} \\
    \midrule
    1     & $[-135,0.9,10,0.4]^{\rm{T}}$    & 1     & 40    & 0.08 \\
    2     & $[-90,0.2,60,0.8]^{\rm{T}}$     & 8     & 33    & 0.08 \\
    3     & $[-45,0.8,20,2.0]^{\rm{T}}$     & 8     & 38    & 0.08 \\
    4     & $[90,100,-1.4,-0.4]^{\rm{T}}$   & 16    & 17    & 0.08 \\
    5     & $[45,-0.3,20,1.6]^{\rm{T}}$     & 16    & 17    & 0.08 \\
    6     & $[0,-0.6,180,-1.2]^{\rm{T}}$    & 24    & 17    & 0.08 \\
    7     & $[135,-0.2,10,3.0]^{\rm{T}}$    & 24    & 9     & 0.08 \\
    8     & $[-90,0.6,140,-1.2]^{\rm{T}}$   & 16    & 25    & 0.08 \\
    \bottomrule
    \end{tabular}%
  \label{tab:addlabel}%
\end{table}%

Additionally, ${\bf{I}_n}$ is $n$-dimension identity matrix.
The transition density is ${f_{k + 1|k}}\left( { \cdot |\xi } \right)\sim {\cal N}\left( { \cdot |{\bf{F}}\xi ,{{\bf{Q}}_{k + 1}}} \right)$, and
\begin{equation}
{\bf{F}} = {{\bf{I}}_2} \otimes \left( {\begin{array}{*{20}{l}}
1&\tau \\
0&1
\end{array}} \right),\quad {\bf{Q}} = q{{\bf{I}}_2} \otimes \left( {\begin{array}{*{20}{c}}
{{\tau ^4}/4}&{{\tau ^3}/2}\\
{{\tau ^3}/2}&{\tau^2}
\end{array}} \right),
\end{equation}
where $\otimes$ is Kronecker product.
Set sampling time $\tau  = 1$, maneuver parameters $q = 8.1 \times {10^{ - 3}}$, ${p_S} = 0.99$, ${P_i} = 250$, pruning threshold ${\Gamma _p} = 4 \times {10^{ - 3}}$, merge threshold ${\Gamma _m} = 4$, birth threshold ${\Gamma _b} = 6.4$, resolution of range and bearing $r = 2.5{\rm{m}}$, $\theta  = {3^\circ }$, FOV $R = \left[ {0,200} \right]\left( \rm{m} \right),\Theta  = \left[ {{0^\circ },{{180}^\circ }} \right]$, OSPA parameter $c = 8,p = 2$ \cite{Schuhmacher2008OSPA}.
Set SNR$=12, 18$dB, ${\sigma _n} = 1.5$, ${\sigma _s} = \left[ {6,12} \right]$. Run MC$=200$ simulations.

\begin{figure}[!tbp]
  \centering
  \subfloat[\label{fig12dBin2Filters}]{
		\includegraphics[width=0.245\textwidth]{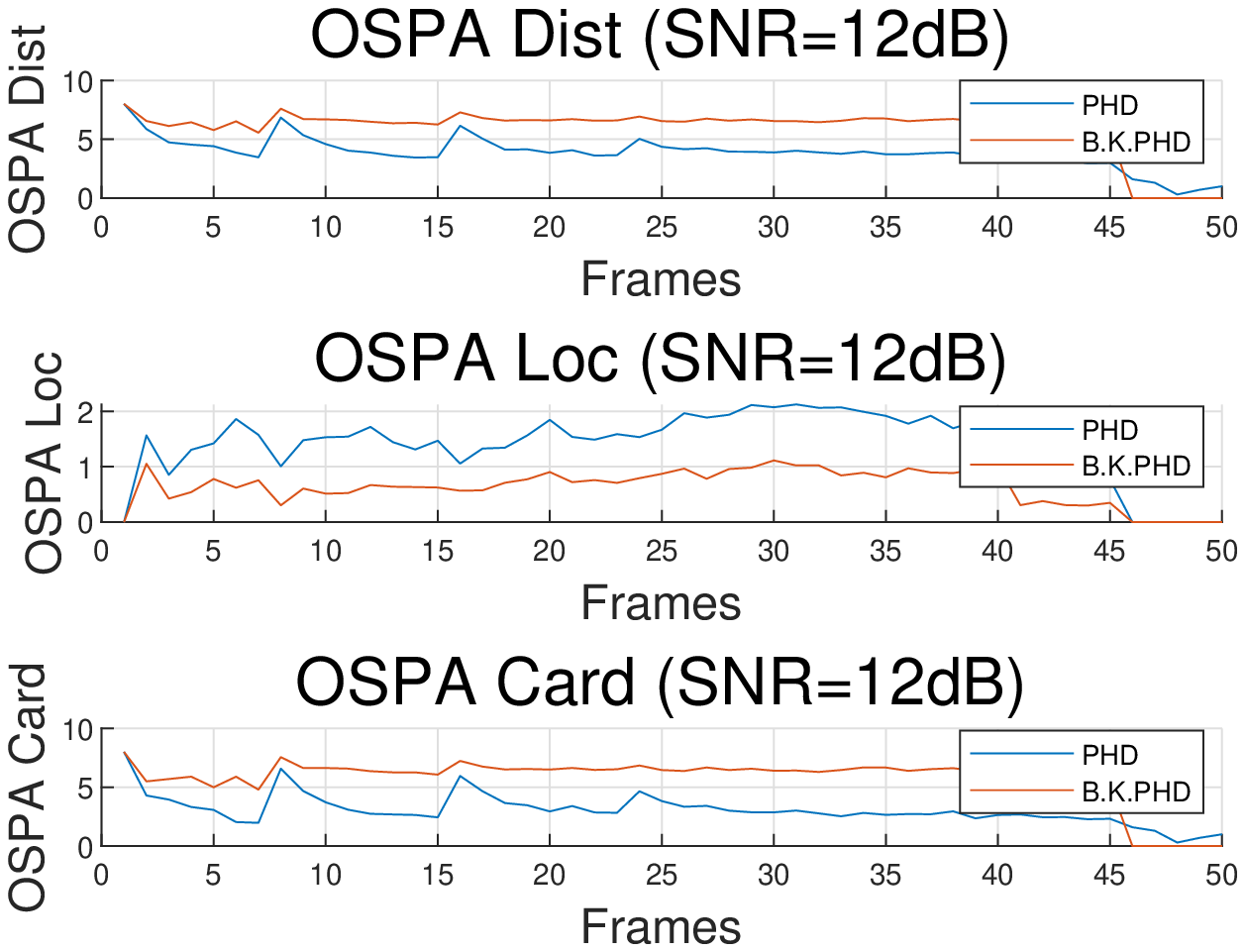}}
  \subfloat[\label{fig18dBin2Filters}]{
		\includegraphics[width=0.245\textwidth]{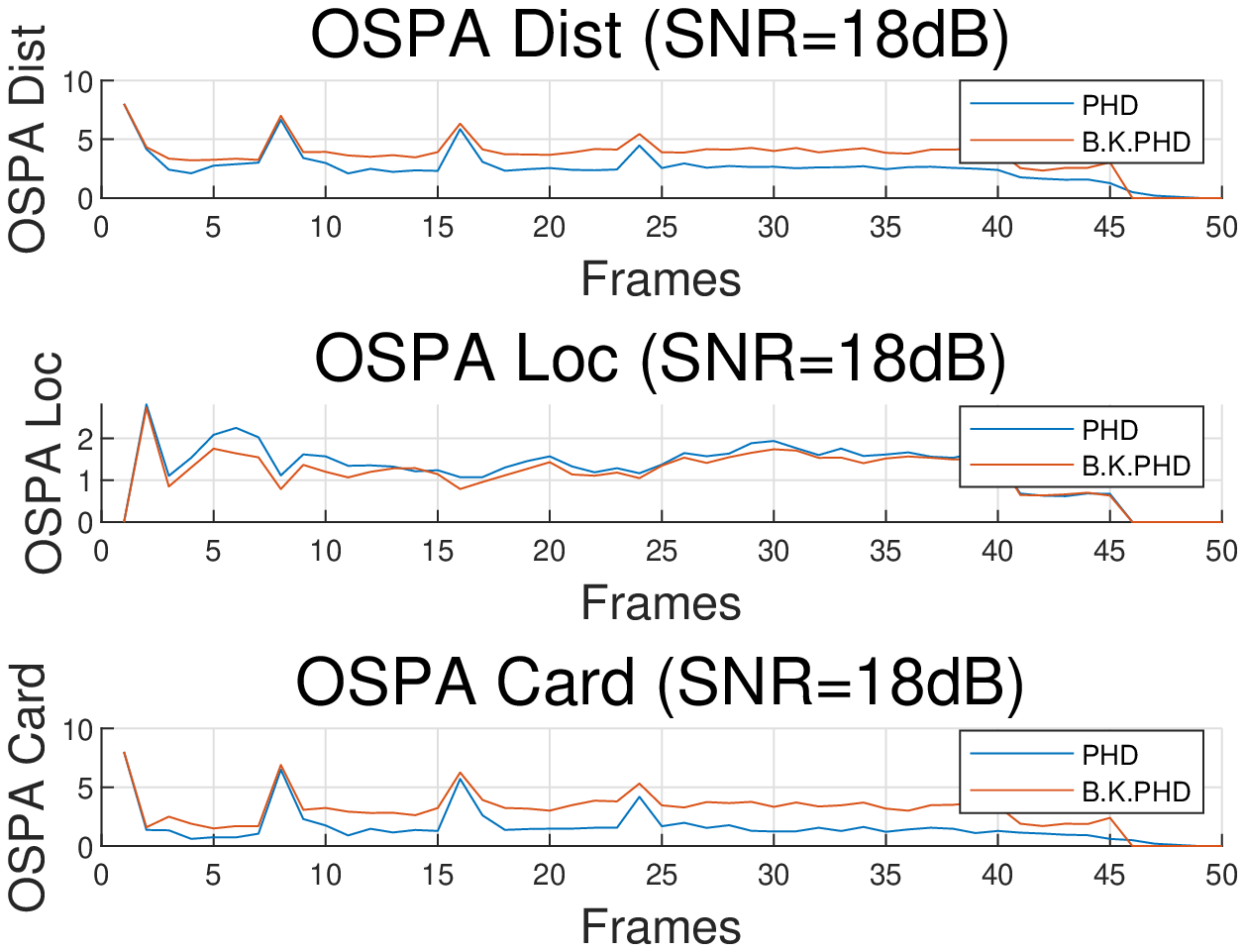}}
  \caption{OSPA analysis of the newly proposed PHD filter and B.K.PHD filter at different SNRs. (a) SNR=12dB, (b) SNR=18dB.} \label{fig1218dBin2Filters}

  \subfloat[\label{figSubVisual12dB_PHD}]{
		\includegraphics[width=0.23\textwidth]{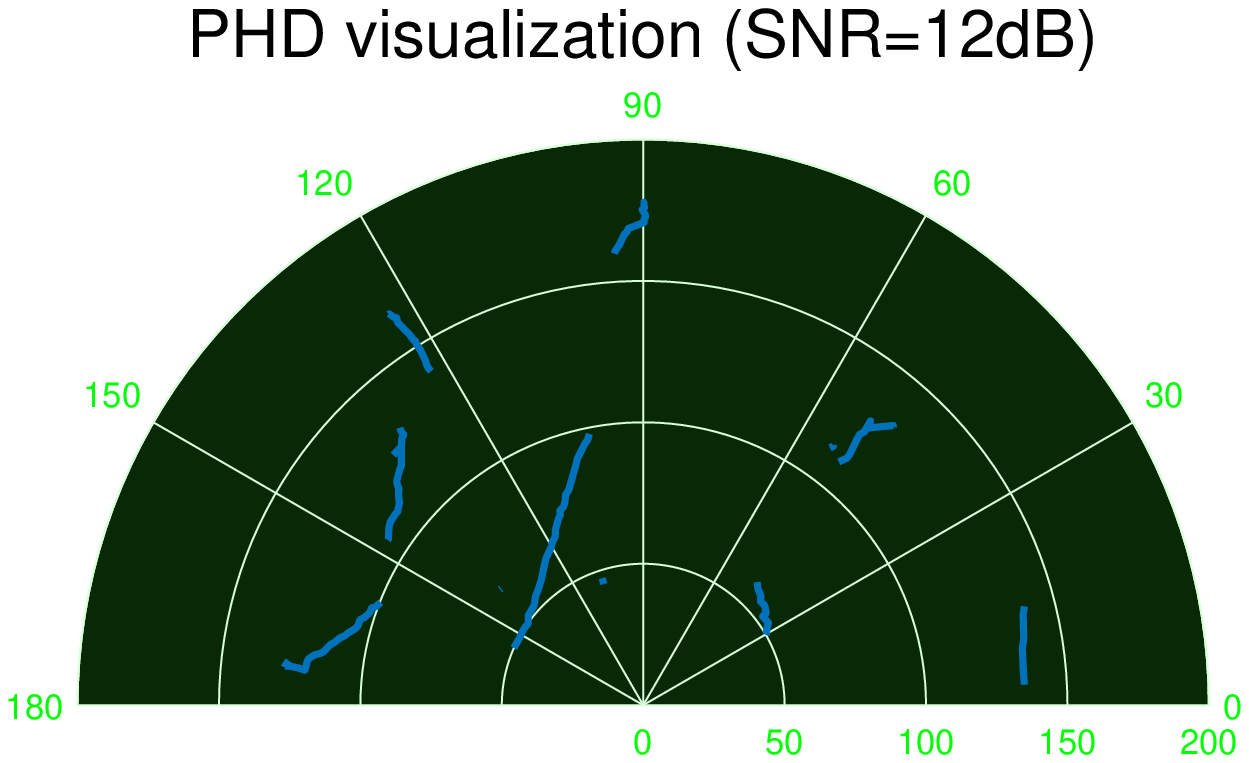}}
  \subfloat[\label{figSubVisual12dB_BKPHD}]{
		\includegraphics[width=0.23\textwidth]{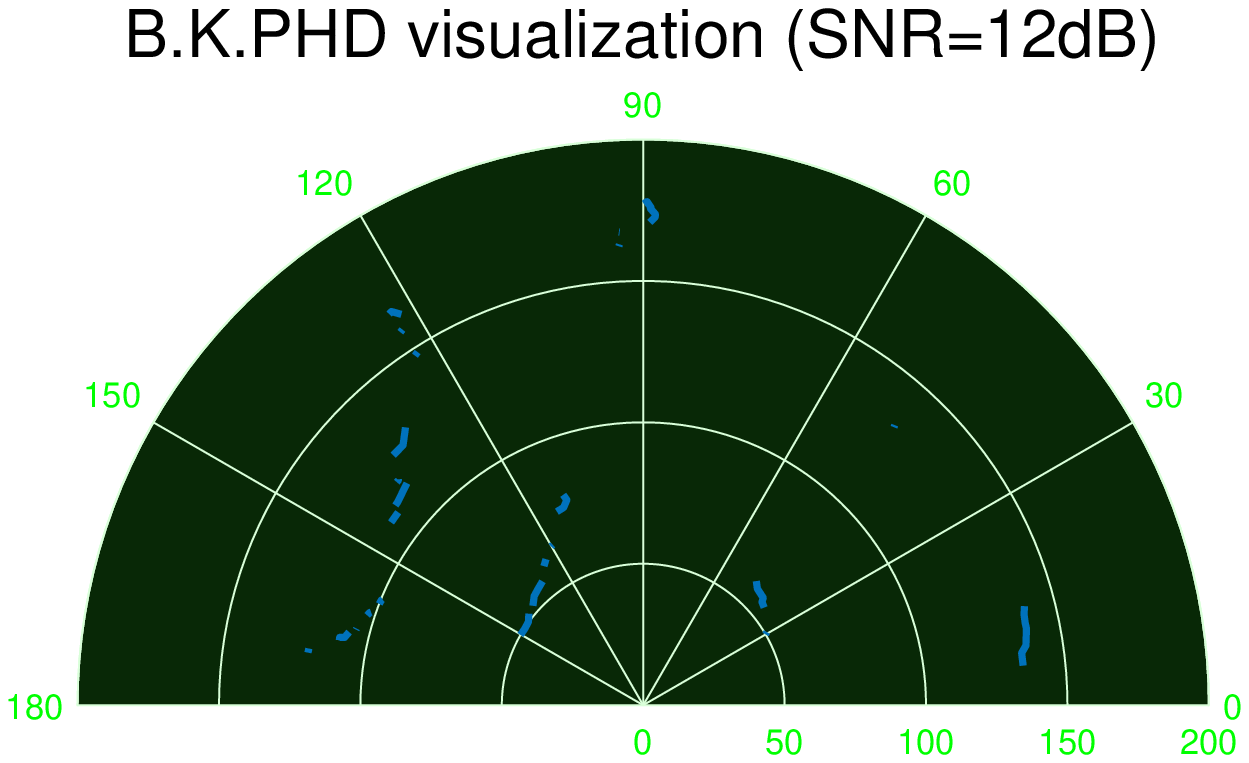}}
  \caption{Tracking visualization of 2 filters at SNR=12dB: (a) the newly proposed PHD filter, (b) the B.K.PHD filter.} \label{figVisual2filters}
\end{figure}


\begin{figure}[!hbtp]
    \begin{minipage}[h]{0.48\textwidth}
    \includegraphics[width=1.0\textwidth]{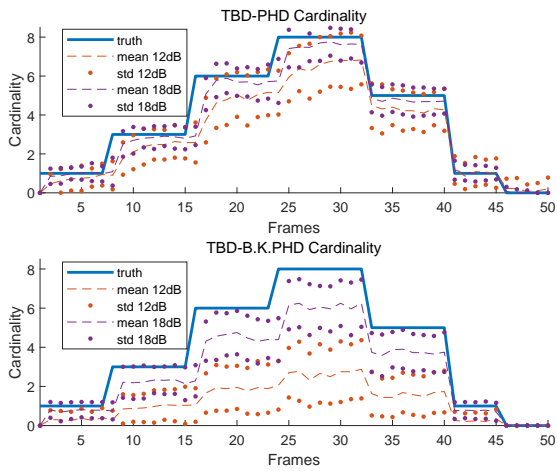}
    \caption{Mean and standard deviation (std) of estimated number by 2 filters.}\label{fig_cd}
    \end{minipage}

\end{figure}

The proposed TBD-PHD filter is contrasted with the B.K.PHD filter.
Because B.K.PHD filter represents a family of TBD-PHD filters that the DBT updating equations are mechanically applied to.
Initial information of targets is shown in Table \ref{table_initTargets}.
OSPA of both filters in SNR=12dB or SNR=18dB is shown in Fig. \ref{fig1218dBin2Filters}.
The tracking result is shown in Fig. \ref{figVisual2filters}.
The mean and std of the estimated number are shown in Fig. \ref{fig_cd}.

Based on OSPA results, the proposed TBD-PHD completely outperforms the B.K.PHD filter.
The location error of the B.K.PHD filter is sometimes lower than that of the novel TBD-PHD filter.
That is because the B.K.PHD filter has significant missed detection.
This results in the mismatch between true targets and estimated targets in the Hungarian assignment.
Thus, the low location error of the B.K.PHD filter do not result from high location accuracy, but the matching failure caused by missing detection.

\section{Conclusion}
Based on KLD minimization, FISST theory and strict Bayes criterion, an analytical and principled closed-form solution of TBD-PHD filter is derived in this paper.
Moreover, Poisson conjugate prior for TBD-PHD filter is proven.
Furthermore, the capping operation is devised, which doesn't change the properties of the conjugate prior.
Finally, the SMC implementation is designed and outperforms previous TBD-PHD filters, whose models and assumptions mismatch.

\bibliographystyle{IEEEtran}
\bibliography{TBD_PHD}
%
%
%

\end{document}